%% file: paper.tex
\documentclass[runningheads]{llncs}
\usepackage[T1]{fontenc}
\usepackage{graphicx}
\usepackage{hyperref,xcolor,amsmath,mathtools,amssymb,booktabs,multirow,ulem}
\usepackage[symbol]{footmisc}

\DeclareMathOperator*{\combine}{\mathrm{Concat}}
\DeclareMathOperator{\softmax}{\mathrm{Softmax}}
\newcommand{\norm}[1]{#1}
\newcommand{\T}{^\mathsf{T}}
\DeclareMathOperator{\mlp}{\mathrm{MLP}}
\newcommand{\R}{\mathbb{R}}
\newcommand{\G}{\mathbf{G}}

\newcommand{\gradecomma}{,}
\renewcommand{\P}{\mathrm{P}}
\newcommand{\C}{\mathrm{C}}
\newcommand{\set}[2]{\{ #1 \hspace{2pt}|\hspace{2pt} #2 \}}
\DeclareMathOperator*{\argmin}{\mathrm{arg} \hspace{2pt} \mathrm{min}}

\newcommand{\revision}[1]{\textcolor{black}{#1}}
\newcommand{\revisionout}[1]{}

\begin{document}
\title{LaB-GATr: geometric algebra transformers for large biomedical surface and volume meshes}
\titlerunning{$\G(3, 0, 1)$ transformers for large biomedical meshes}
%
\author{Julian Suk\inst{*}\orcidID{0000-0003-0729-047X} \and
Baris Imre\inst{*}\orcidID{0009-0005-3724-699X} \and\\
Jelmer M. Wolterink\orcidID{0000-0001-5505-475X}}
\authorrunning{J. Suk et al.}
%
\institute{Department of Applied Mathematics, Technical Medical Centre, University of Twente, Enschede, The Netherlands\\
\email{\{j.m.suk, j.m.wolterink\}@utwente.nl}}
%
\maketitle              
\begin{abstract}
Many anatomical structures can be described by surface or volume meshes. Machine learning is a promising tool to extract information from these 3D models. However, high-fidelity meshes often contain hundreds of thousands of vertices, which creates unique challenges in building deep neural network architectures. Furthermore, patient-specific meshes may not be canonically aligned which limits the generalisation of machine learning algorithms. We propose LaB-GATr, a transfomer neural network with geometric tokenisation that can effectively learn with large-scale (bio-)medical surface and volume meshes through sequence compression and interpolation. Our method extends the recently proposed geometric algebra transformer (GATr) and thus respects all Euclidean symmetries, i.e. rotation, translation and reflection, effectively mitigating the problem of canonical alignment between patients. LaB-GATr achieves state-of-the-art results on three tasks in cardiovascular hemodynamics modelling and neurodevelopmental phenotype prediction, featuring meshes of up to 200,000 vertices. Our results demonstrate that LaB-GATr is a powerful architecture for learning with high-fidelity meshes which has the potential to enable interesting downstream applications. Our implementation is publicly available.
\textcolor{white}{\footnote[1]{Equal contribution.}}\footnote[2]{\href{https://github.com/sukjulian/lab-gatr}{\texttt{github.com/sukjulian/lab-gatr}}}

\end{abstract}
\section{Introduction}
Deep neural networks can leverage biomedical data to uncover previously unknown cause-and-effect relations~\cite{VosyliusWang2020} and enable novel ways of medical diagnosis and treatment~\cite{LiWang2021,SarasuaPoelsterl2021}. There has been active research into using deep neural networks for biomedical modelling, such as cardiovascular biomechanics estimation~\cite{ArzaniWang2022} and neuroimage analysis based on the cortical surface~\cite{ZhaoWu2023}. These applications feature 3D mesh representations of patient anatomy. Depending on the downstream application, 3D meshes either discretise the surface of an organ or vessel with, e.g., triangles, or the interior with, e.g., tetrahedra. Graph neural networks (GNN) have gained traction in the context of learning with 3D meshes due to their direct applicability, flexibility regarding mesh size, and good performance on several benchmarks~\cite{FawazWilliams2021,SukBrune2023,SukHaan2022}. However, GNNs are known to suffer from over-squashing, i.e. the loss of accuracy when compressing exponentially growing information into fixed-sized channels when propagating messages on long paths~\cite{AlonYahav2021}. This makes them inefficient at accumulating large enough receptive fields, capable of learning global interactions across meshes. PointNet++~\cite{QiYi2017} can circumvent this issue to some extent via pooling to coarser surrogate graphs, but this requires problem-specific hyperparameter setup and might still fail around bottlenecks.
In contrast, the transformer architecture~\cite{VaswaniShazeer2017}, treats its input as a sequence of tokens (or patches) and models all pair-wise interactions, thus aggregating global context after a single layer. In medical imaging, transformers have been successfully applied in the context of, e.g., semantic segmentation of 2D histopathology images~\cite{JiZhang2021} and 3D brain tumor magnetic resonance images (MRI)~\cite{HatamizadehNath2022}. Nevertheless, applications to 3D biomedical meshes remain scarce due to the difficulty of finding a unified framework that addresses their sheer size, commonly in the hundreds of thousands of vertices. Dahan et al.~\cite{DahanFawaz2022,DahanFawaz} addressed this by aligning 3D cortical surface meshes across subjects via morphing to an icosphere, which allows for segmentation into coarser triangular patches.

The recently proposed geometric algebra transformer (GATr)~\cite{BrehmerHaan2023} has achieved state-of-the-art accuracy on several geometric tasks, largely due to the incorporation of task-specific symmetries ($\mathrm{SE}(3)$-equivariance) and modelling of interactions via geometric algebra.\footnote{We refer the interested reader to \cite{BrandstetterBerg2023,RuheGupta2023,RuheBrandstetter2023}.}
Even though GATr uses memory-efficient attention~\cite{RabeStaats} with linear complexity, large 3D meshes still cause it to exceed GPU memory during training. In the context of graph transformers, the common bottleneck of GPU memory has been addressed by (1) sparse, local attention mechanisms together with expander graph rewiring~\cite{DeacLackenby2022,ShirzadVelingker2023}, as well as (2) vertex clustering together with learned graph pooling and upsampling~\cite{JannyBeneteau2023,KongChen2023}.

In this work, we scale GATr to \textbf{la}rge-scale \textbf{b}iomedical (LaB-GATr) surface and volume meshes. Since each of these meshes discretises a continuous shape, mesh connectivity should be treated as an artefact, not a feature. Thus, we opt for a vertex clustering approach to make self-attention tractable. We derive a general-purpose tokenisation algorithm and interpolation method with learned feature representations in geometric algebra space, which retains all symmetries of GATr as well as the original mesh resolution, while decreasing the number of active tokens used in self-attention. We demonstrate the efficacy of our method on three tasks involving biomedical meshes. LaB-GATr sets a new state-of-the-art in prediction of blood velocity in high-fidelity, synthetic coronary artery meshes~\cite{SukBrune2023} which would exceed 48 GB of VRAM when using GATr with the same number of parameters. Furthermore, Lab-GATr excels at phenotype prediction of cortical surfaces from the \href{http://www.developingconnectome.org/}{Developing Human Connectome Project} (dHCP)~\cite{EdwardsRueckert2022}, setting a new state-of-the-art in postmenstrual age (PMA) estimation without morphing the cortical surface mesh to an icosphere. We provide a modular implementation of Lab-GATr and an interface with which the geometric algebra back-end can be treated as black box.

\begin{figure}[t]
	\includegraphics[width=\columnwidth]{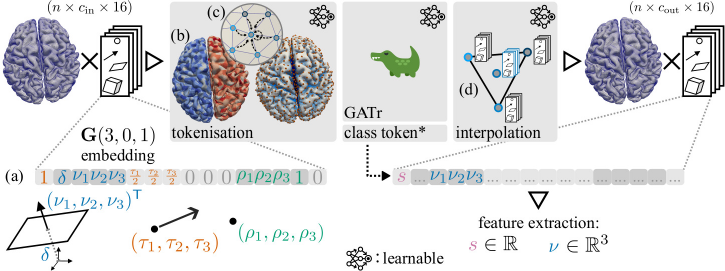}
	\caption{\textbf{LaB-GATr} takes input features in the form of multivectors, which are constructed by embedding, e.g., mesh vertices as points and surface normal vectors as planes (see Section~\ref{sec:embed}). In the tokenisation module, the features are pooled to a coarse subset of mesh vertices via message passing (see Section~\ref{sec:pool}). The tokenisation allows control over the number of tokens which are processed by the GATr module. Downstream, the interpolation module lifts the tokenisation back to original mesh resolution (see Section~\ref{sec:interp}). An optional ($^*$) class token is appended to the token sequence for mesh-level output. Subsequently, scalar or vector-valued output features are extracted.}
	\label{fig:zero}
\end{figure}

\section{Methods}
Fig.~\ref{fig:zero} provides an overview of our method. We extend GATr by geometric tokenisation for scaling to high-fidelity meshes. In the following, we discuss background on geometric algebra and transformers before introducing our contribution.

\subsection{Geometric algebra}
GATr is built on the geometric algebra $\G(3, 0, 1)$ which uses projective geometry: a fourth coordinate is appended to 3D space, so translation can be expressed as linear maps. At the core of geometric algebra lies the introduction of an associative (but not commutative) geometric product of vectors \revision{$y, z$}, simply denoted as \revision{$yz$}. Given a 4D orthogonal basis $\{e_i\}_i$, it holds that $e_0 e_0 = 0$, $e_i e_i = 1$ $(i \neq 0)$, and $e_i e_j = -e_j e_i$ $(i \neq j)$. All possible, linearly independent geometric products span a 16-dimensional vector space of multivectors $x \in \G(3, 0, 1)$:
\begin{equation}\label{eq:multivector}
	x = (x_s\gradecomma
	\underbrace{x_0, x_1, x_2, x_3}_\text{vectors}\gradecomma
	\underbrace{x_{01}, x_{02}, x_{03}, x_{12}, x_{13}, x_{23}}_\text{bivectors}\gradecomma
	\underbrace{x_{012}, x_{013}, x_{023}, x_{123}}_\text{trivectors}\gradecomma
	x_{0123}
	)
\end{equation}
which can represent geometric quantities like points and planes.

\subsection{Embedding meshes in $\G(3, 0, 1)$}\label{sec:embed}
Consider an arbitrary surface or volume mesh consisting of $n$ vertices. For each vertex, we construct a $d$-dimensional positional encoding which describes its unique geometric properties within this mesh. Each positional encoding is composed of a set of $c$ geometric objects, e.g. the surface normal vector (for surface meshes) or the (scalar) distance to the surface (for volume meshes). Table~\ref{tab:embedding} provides a look-up on how to embed relevant geometric objects, such as points and planes, as multivectors $x \in \G(3, 0, 1)$ (see Fig.~\ref{fig:zero}a). Consequently, we describe each mesh by a tensor $X^{(0)} \in \R^{n \times d}$ with $d = c \cdot 16$.

\begin{table}[t]
	\begin{center}
		\caption{Embedding of some common geometric objects as $x \in \G(3, 0, 1)$. See \eqref{eq:multivector} for the 16 multivector components. Other multivector components remain zero. In geometric algebra, geometric operations can be multivectors just like geometric objects.}
		\begin{tabular}{@{}lrcl@{}}
			\toprule
			Geometric object / operation & \multicolumn{3}{r}{Multivector mapping}\\
			\midrule
			Scalar $s \in \R$ & $x_s$ & $=$ & $s$\\
			Plane with normal $\nu \in \R^3$ and offset $\delta \in \R$ & $(x_0, x_1, x_2, x_3)$ & $=$ & $(\delta, \nu)$\\
			Point $\rho \in \R^3$ & $(x_{012}, x_{013}, x_{023}, x_{123})$ & $=$ & $(\rho, 1)$\\
			\midrule
			Translation $\tau \in \R^3$ & $(x_s, x_{01}, x_{02}, x_{03})$ & $=$ & $(1, \frac{1}{2} \tau)$\\
			\bottomrule
			\label{tab:embedding}
		\end{tabular}
	\end{center}
\end{table}

\subsection{Geometric algebra transformers}
Given a tensor of input features $X^{(l)}$, we define transformer blocks as follows:
\begin{align*}
	A^{(l)} &= X^{(l)} + \xi\left(\combine_h \hspace{2pt} \softmax\left( \frac{q_h(\norm{X^{(l)}}) k_h(\norm{X^{(l)}})\T}{\sqrt{d}} \right) v_h(\norm{X^{(l)}})\right)\\
	X^{(l + 1)} &= A^{(l)} + \mlp(\norm{A^{(l)}}).
\end{align*}
As is common practice, $q_h, k_h, v_h \colon \R^{n \times d} \to \R^{n \times d}$ consist of layer normalisation composed with learned linear maps. Multi-head self-attention~\cite{VaswaniShazeer2017} over heads indexed by $h$ is implemented via concatenation followed by a learned linear map $\xi$. GATr~\cite{BrehmerHaan2023} introduced layer normalisation, linear and nonlinear maps $\G(3, 0, 1)^{n \times c} \to \G(3, 0, 1)^{n \times c}$, and gated activation functions which can be used to construct $\xi$, $q_h, k_v, v_h$, and $\mlp$ in geometric algebra. GATr blocks are equivariant under rotations, translations and reflections $\rho \in \mathrm{E}(3)$ of the input geometry encoded in $X^{(l)}$, i.e. they map $\rho X^{(l)} \mapsto \rho X^{(l + 1)}$.

\subsection{Learned, geometric tokenisation}
In the following, we introduce tokenisation layers for large surface and volume meshes, allowing us to compress the token sequence for the transformer blocks.

\subsubsection{Pooling}\label{sec:pool}
From the point cloud $\P_\text{fine}$ consisting of the $n$ mesh vertices, we compute a subset of points $\P_\text{coarse}$ via farthest point sampling. We partition the mesh vertices into tokens by assigning each $p \in \P_\text{coarse}$ the (disjoint) cluster of closest points in $\P_\text{fine}$ (see Fig.~\ref{fig:zero}b):
\[
\C(p) = \set{v \in \P_\text{fine}}{p = \argmin_{q \in \P_\text{coarse}} \Vert v - q \Vert_2}.
\]
This means that each point in $\P_\text{fine}$ is clustered with the point in $\P_\text{coarse}$ to which it is the closest. Define $\revision{n_\text{coarse}} = |\P_\text{coarse}|$. Given a tensor of input features $X \in \R^{n \times d}$, we perform learned message passing within these clusters (see Fig.~\ref{fig:zero}c) as follows:
\begin{align*}
	m_{v \to p} &= \mlp(X^{(0)}|_v, \hspace{2pt} p - v) && (\text{message from} \hspace{2pt} \P_\text{fine} \hspace{2pt} \text{to} \hspace{2pt} \P_\text{coarse})\\
	X^{(1)}|_p &= \frac{1}{|\C(p)|} \sum_{v \in C(p)} \hspace{2pt} m_{v \to p} && (\text{aggregation of} \hspace{2pt} X^{(1)} \in \R^{\revision{n_\text{coarse}} \times d})
\end{align*}
where $X|_q$ denotes the row of $X$ corresponding to point $q$. This layer maps $\R^{n \times d} \to \R^{\revision{n_\text{coarse}} \times d}$ with $\revision{n_\text{coarse}} < n$. Within this framework, we can embed $p - v$ as translation (see Table~\ref{tab:embedding}) and use the \revision{multilayer perceptron ($\mlp$)} introduced by \cite{BrehmerHaan2023} to reduce the number of tokens in a way that is fully compatible with GATr. In particular, this layer respects all symmetries of $\G(3, 0, 1)$.

\subsubsection{Interpolation}\label{sec:interp}
Given a tensor $X^{(l)} \in \R^{\revision{n_\text{coarse}} \times d}$ we define learned\revisionout{, barycentric} interpolation to the original mesh resolution $Y \in \R^{n \times d}$ as follows:
\begin{align*}
	X^{(l + 1)}|_v &= \frac{\sum_p \lambda_{p, v} \hspace{2pt} X^{(l)}|_p}{\sum_p \lambda_{p, v}}, \hspace{16pt} \lambda_{p, v} \coloneqq \frac{1}{\Vert p - v \Vert_2^2 + \revision{\epsilon}},\\
	Y &= \mlp(X^{(l + 1)}, X^{(0)})
\end{align*}
where for each $v \in \P_\text{fine}$ we sum over the three (four) closest points $p \in \P_\text{coarse}$ for surface (volume) meshes (see Fig.~\ref{fig:zero}d) \revision{and $\epsilon$ is a small constant}. This layer lifts the tokenisation and ensures neural-network output in original mesh resolution. The \revisionout{barycentric} interpolation provably behaves as expected in $\G(3, 0, 1)$ (see appendix) \revision{and by using the $\mlp$ introduced by \cite{BrehmerHaan2023} this layer respects all symmetries of $\G(3, 0, 1)$}.

\subsection{Neural network architecture}\label{sec:architecture}
LaB-GATr is composed of geometric algebra embedding, tokenisation module, GATr module, and interpolation module followed by feature extraction (see Fig.~\ref{fig:zero}). We embed the input mesh in $\G(3, 0, 1)$ (see Section~\ref{sec:embed}) based on an application-specific set of geometric descriptors. The embedding $X^{(0)} \in \R^{n \times d}$ is pooled (see Section~\ref{sec:pool}) and the resulting $\revision{n_\text{coarse}}$ tokens are fed into a GATr module. For mesh-level tasks, we append a global class token which we embed as the mean of the $\revision{n_\text{coarse}}$ tokens. For vertex-level tasks, the $m$ output tokens of the transformer are interpolated (see Section~\ref{sec:interp}) back to the original mesh resolution. For classification tasks (like segmentation), $\softmax$ can be used over the channel dimension at the mesh- or vertex-level output.

\begin{figure}[t]
	\includegraphics[width=\columnwidth]{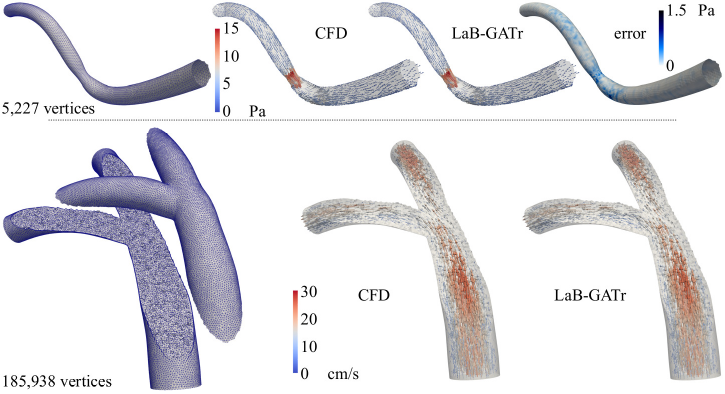}
	\caption{Qualitative results of cardiovascular hemodynamics estimation on test-split arteries. Left: (top) surface and (bottom) volume mesh. Middle: (top) wall shear stress and (bottom) \revision{velocity field, via computational fluid dynamics (CFD)}. Right: LaB-GATr prediction.}
	\label{fig:cardiovascular}
\end{figure}

\section{Experiments and results}
We evaluate LaB-GATr on three tasks previously explored with GNNs and transformers. Since all of these are regression tasks, we train LaB-GATr under $\mathrm{L^1}$ loss using the Adam optimiser~\cite{KingmaBa2015} and learning rate $3\mathrm{e}{-4}$ with exponential decay on \textsc{Nvidia} L40 (48 GB) GPUs. We use the same number of channels and attention heads throughout all experiments, leading to around 320k trainable parameters.

\begin{table*}[t]
	\caption{Comparison of LaB-GATr with state-of-the-art baselines. For WSS and hemodynamics estimation, we report mean $\varepsilon$ (lower is better) across the test set. For PMA \revisionout{and GA} estimation, some of the referenced works report lowest and some average $\mathrm{MAE}$ across three training runs. We report lowest $\mathrm{MAE}$ in brackets where applicable.}
	\centering
	\begin{tabular}{@{}lllccc@{}}
		\toprule
		Domain & Dataset & Model & \multicolumn{2}{c}{Disparity} & Metric\\
		&&& (average) & (lowest) &\\
		\midrule
		\multirow{8}{*}{Cardiovascular} & \multirow{6}{*}{WSS~\cite{SukHaan2022}} & GEM-CNN~\cite{SukHaan2022_} & \multicolumn{2}{c}{7.8} & \multirow{6}{*}{$\varepsilon$ [\%]}\\
		&& GATr~\cite{BrehmerHaan2023} & \multicolumn{2}{c}{\textbf{5.5}} & \\
		&& \textit{LaB-GATr} & \multicolumn{2}{c}{\textbf{\hphantom{$^*$}5.5$^*$}} & \\
		\cmidrule{3-5}
		&& \textit{LaB-GATr}$^1$ & \multicolumn{2}{c}{7.0} & \\
		&& \textit{LaB-GATr}$^2$ & \multicolumn{2}{c}{12.3} & \\
		\cmidrule{2-5}
		& \multirow{2}{*}{Velocity~\cite{SukBrune2023}} & SEGNN~\cite{SukBrune2023} & \multicolumn{2}{c}{7.4} & \\
		&& \textit{LaB-GATr} & \multicolumn{2}{c}{\textbf{3.\revision{3}}} & \\
		\midrule
		\multirow{4}{*}{Neuroimaging} & \multirow{4}{68pt}{PMA~\cite{EdwardsRueckert2022} (native)} & SiT~\cite{DahanFawaz2022} & – & (0.68) & \multirow{4}{*}{$\mathrm{MAE}$ [weeks]}\\
		&& MS-SiT~\cite{DahanFawaz} & 0.59 & – & \\
		&& \cite{UnyiGyiresToth2022} & – & (0.54)& \\
		&& \textit{LaB-GATr} & \textbf{0.54} & (0.52) & \\
		\bottomrule
		\multicolumn{6}{l}{$^*$with 10-fold compression \hspace{2pt} $^1$static pooling module $^2$static interpolation module}
		\label{tab:quantitative}
	\end{tabular}
\end{table*}

\subsection{Coronary artery models}
\subsubsection{Surface-based WSS estimation}
The publicly available dataset~\cite{SukHaan2022} consists of 2,000 synthetic, coronary artery surface meshes (around 7\revision{,000} vertices) with simulated, steady-state wall shear stress (WSS) vectors. We chose $\revision{n_\text{coarse}} = 0.1 \cdot n$ in the tokenisation (see Section~\ref{sec:pool}) and embed vertex positions as points (see Table~\ref{tab:embedding}), surface normal vectors as oriented planes, and geodesic distances to the artery inlet as scalars. The GATr back-end (see Section~\ref{sec:architecture}) was set up identical to \cite{BrehmerHaan2023}. We trained LaB-GATr on a single GPU for 4,000 epochs (58 $\frac{\text{s}}{\text{epoch}}$) with batch size 8 on an 1600:200:200 split of the dataset. Fig.~\ref{fig:cardiovascular} shows an example of LaB-GATr prediction on a test-split mesh. Table~\ref{tab:quantitative} shows approximation error $\varepsilon$~\cite{SukHaan2022_,BrehmerHaan2023} of LaB-GATr compared to the baselines. LaB-GATr matches GATr's accuracy $\varepsilon = 5.5 \%$ \revision{(standard deviation was $\pm 2.0 \%$ over 200 test cases)} despite its 10-fold compression of the token sequence. Ablation of the \revision{learnable parameters} of \revision{the} pooling and interpolation module reveals that \revision{learned interpolation} plays a bigger role in performance than \revision{learned pooling}.

\subsubsection{Volume-based velocity field estimation}
The publicly available dataset~\cite{SukBrune2023} consists of 2,000 synthetic, bifurcating coronary artery volume meshes (around 175\revision{,000} vertices) with simulated, steady-state velocity vectors. We chose $\revision{n_\text{coarse}} = 0.01 \cdot n$ for the tokenisation and embed vertex positions as points, directions to the artery inlet, outlets, and wall as oriented planes, and distances to the artery inlet, outlets, and wall as scalars. The size of these volume meshes caused GATr to exceed memory of an \textsc{Nvidia} L40 (48 GB) GPU. Leveraging our tokenisation enabled training LaB-GATr with batch size 1. We trained LaB-GATr on four GPUs in parallel for 300 epochs (10:24 $\frac{\text{min}}{\text{epoch}}$) on an 1600:200:200 split of the dataset. Fig.~\ref{fig:cardiovascular} shows an example of LaB-GATr prediction on a test-split mesh. Table~\ref{tab:quantitative} compares LaB-GATr to \cite{SukBrune2023}. LaB-GATr sets a new state-of-the-art for this dataset with $\varepsilon = 3.\revision{3} \%$ \revision{(standard deviation was $\pm 4.4 \%$ over 200 test cases)} compared to the previous $7.4 \%$.

\subsection{Postmenstrual age prediction from the cortical surface}
The publicly available third release of dHCP~\cite{EdwardsRueckert2022} consists of 530 \revision{newborns'} cortical surface meshes (81,924 vertices each) which are symmetric across hemispheres. We estimate the subjects' postmenstrual age (PMA) at the time of scan. We chose $\revision{n_\text{coarse}} = 0.024 \cdot n$ for the tokenisation and embed vertex positions as points, surface normal vectors as oriented planes, the reflection planes between symmetric vertices as oriented planes, and myelination, curvature, cortical thickness, and sulcal depth as scalars. Since we observed \revision{quick} convergence on the validation split, we trained LaB-GATr on a single GPU for only 200 epochs (1:38 $\frac{\text{min}}{\text{epoch}}$) with batch size 4 on the 423:53:54 splits used in \cite{DahanFawaz}.
Figure~\ref{fig:neuroimaging} shows a Bland-Altman plot and the cortical surface of two subjects, indicating good accuracy. Table~\ref{tab:quantitative} shows mean absolute error ($\mathrm{MAE}$) and comparison to the baselines. In contrast to all baselines, LaB-GATr runs directly on the cortical surface mesh without morphing to a sphere. LaB-GATr sets a new state-of-the-art in PMA estimation on "native space"~\cite{FawazWilliams2021} with average $\mathrm{MAE} = 0.54$ weeks \revision{(standard deviation was $\pm 0.39$ weeks over 54 test cases times three runs)}.

\begin{figure}[t]
	\includegraphics[width=\columnwidth]{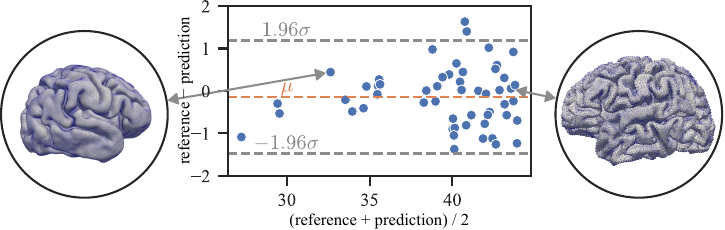}
	\caption{Postmenstrual age (PMA) prediction \revision{(values in weeks)} on test-split subjects. LaB-GATr predictions based on the newborn's cortical surface correlated well with the reference values.}
	\label{fig:neuroimaging}
\end{figure}

\section{Discussion and conclusion}
In this work, we propose LaB-GATr, a general-purpose geometric algebra transformer for large-scale surface and volume meshes, often found in biomedical engineering. We extend GATr by learned tokenisation and interpolation in geometric algebra. Our method \revision{can be understood as} a thin PointNet++~\cite{QiYi2017} wrapper which is adapted to projective geometric algebra. Through the self-attention mechanism, LaB-GATr models global interactions within the mesh, while avoiding the over-squashing phenomenon exhibited by GNNs. Notably, LaB-GATr is equivariant to rotations, translations, and reflections of the input mesh, thus circumventing the problem of cortical surface alignment. In our experiments, LaB-GATr matched the performance of GATr~\cite{BrehmerHaan2023} on the same dataset, suggesting near loss-less compression. \revision{Even though LaB-GATr introduces additional, trainable parameters over GATr, we found that computing self-attention between less tokens outweighed the parameter overhead and led to favourable training times.} Beside estimating PMA of subjects from dHCP, we also attempted gestational age (GA) estimation with LaB-GATr. However, we observed considerably \revision{lower accuracy}\revisionout{, the cause of which we are currently investigating}. \revision{We believe that GA can be best explained by vertex-specific biomarkers such as myelination, which was provided on sphericalised and subsequently sub-sampled brains and back-projection to the cortical surface erased some of their spatial context.}

\revision{In  theory}, our geometric pooling scheme introduces all necessary building blocks to define patch merging and to build a geometric version of sliding window (\revision{Swin}) attention~\cite{LiuLin2021}, which is an interesting direction for future work. We believe that Lab-GATr has potential as general-purpose model for learning with large (bio-)medical surface and volume meshes, enabling interesting downstream applications. In particular, we are interested to explore brain parcellation and attention-map-based analysis of biomedical pathology.  
Geometric algebra introduces an inductive bias to our learning system. In future work, we aim to investigate to what extent this affects LaB-GATr predictions and derive theoretical guarantees.


\begin{credits}
\subsubsection{\ackname} This work is funded in part by the 4TU Precision Medicine programme supported by High Tech for a Sustainable Future, a framework commissioned by the four Universities of Technology of the Netherlands. Jelmer M. Wolterink was supported by the NWO domain Applied and Engineering Sciences Veni grant (18192).

We would like to thank Simon Dahan for his efforts in providing the dHCP dataset as well as Pim de Haan and Johann Brehmer for the fruitful discussions about geometric algebra.

\subsubsection{\discintname}
The authors have no competing interests to declare that are relevant to the content of this article.
\end{credits}
%
%
%
\bibliographystyle{splncs04}
\bibliography{bibliography}

\subsubsection{Appendix}
\input{supplementary_material}

\end{document}

%% file: supplementary_material.tex
\begin{proposition}
	Consider a set of multivectors $x^i \in \mathbf{G}(3, 0, 1)$ and denote by
	\[
	t(x^i) \coloneqq \frac{1}{x^i_{123}}
	\begin{pmatrix}
		x^i_{012}\\
		x^i_{013}\\
		x^i_{023}
	\end{pmatrix}
	\]
	the extraction of point coordinates from a multivector. Assume $x^i_{123} > 0$. Then the point extracted from convex combination of $x^i$ is an element of the convex hull of $\{t(x^i)\}_i$.
\end{proposition}
\begin{proof}
	Let $w = \sum_i \omega_i x^i$, such that $\omega_i > 0$ and $\sum_i \omega_i = 1$. Then
	\[
	t(w) = \frac{1}{w_{123}}
	\begin{pmatrix}
		w_{012}\\
		w_{013}\\
		w_{023}
	\end{pmatrix} = \frac{1}{\sum_i \omega_i x^i_{123}}
	\begin{pmatrix}
		\sum_i \omega_i x^i_{012}\\
		\sum_i \omega_i x^i_{013}\\
		\sum_i \omega_i x^i_{023}
	\end{pmatrix}.
	\]
	Define
	\[
	\omega'_i = \frac{\omega_i x^i_{123}}{\sum_j \omega_j x^j_{123}}
	\]
	and note that $\omega'_i > 0$ and $\sum_i \omega'_i = 1$. Now
	\begin{align*}
		\begin{pmatrix}
			\sum_i \frac{\omega_i}{\sum_j \omega_j x^j_{123}} x^i_{012}\\
			\sum_i \frac{\omega_i}{\sum_j \omega_j x^j_{123}} x^i_{013}\\
			\sum_i \frac{\omega_i}{\sum_j \omega_j x^j_{123}} x^i_{023}
		\end{pmatrix} &=
		\begin{pmatrix}
			\sum_i \frac{\omega'_i}{x^i_{123}} x^i_{012}\\
			\sum_i \frac{\omega'_i}{x^i_{123}} x^i_{013}\\
			\sum_i \frac{\omega'_i}{x^i_{123}} x^i_{023}
		\end{pmatrix}\\
		&= \sum_i \omega'_i t(x^i)
	\end{align*}
	and thus
	\[
	t(w) = t(\sum_i \omega_i x^i) = \sum_i \omega'_i t(x^i).
	\]
	Since the convex hull of a set of points is defined as the set of all convex combinations of its elements, $t(\sum_i \omega_i x^i)$ is an element of the convex hull of $\{t(x^i)\}_i$.
\end{proof}

\paragraph{On the choice of sub-sampling ratio.}
We sub-sampled the coronary volume meshes aggressively to make LaB-GATr fit on an NVIDIA L40 (48 GB). We sub-sample the cortical meshes to the same number of tokens. We increased the sampling ratio for the coronary surface meshes while still affording the same batch size used by GATr.
